\documentclass[11pt]{article}

\usepackage{fullpage}
\usepackage{etoolbox}
\usepackage{enumerate}
\usepackage{qtree}
\usepackage{natbib}
\usepackage[hidelinks]{hyperref}

\usepackage{amsmath,amsfonts,amsthm,amssymb,bm}

\def\eqref#1{equation~\ref{#1}}

\def\1{\bm{1}}

\def\eps{{\epsilon}}

\def\va{{\mathbf{a}}}

\def\vc{{\mathbf{c}}}

\def\vs{{\mathbf{s}}}

\def\vx{{\mathbf{x}}}

\def\vz{{\mathbf{z}}}

\DeclareMathAlphabet{\mathsfit}{\encodingdefault}{\sfdefault}{m}{sl}
\SetMathAlphabet{\mathsfit}{bold}{\encodingdefault}{\sfdefault}{bx}{n}

\def\cB{{\mathcal{B}}}

\def\cP{{\mathcal{P}}}

\def\cY{{\mathcal{Y}}}
\def\cZ{{\mathcal{Z}}}

\newcommand{\E}{\mathbb{E}}

\newcommand{\softmax}{\operatorname{softmax}}

\DeclareMathOperator*{\argmax}{arg\,max}

\newcommand{\abs}[1]{\left\lvert#1\right\rvert}
\newcommand{\norm}[1]{\left\lVert#1\right\rVert}

\theoremstyle{definition}

\newtheorem{theorem}{Theorem}

\usepackage{mystyle}

\usepackage{tcolorbox}
\usepackage{bold-extra}

\usepackage[a4paper,margin=1in]{geometry}
\usetikzlibrary{automata,positioning}
\lstdefinestyle{lstStyle}{numbersep=9pt, tabsize=2, showspaces=false, showstringspaces=false}
\def\indented#1{\list{}{}\item[]}
\let\indented=\endlist

\tcbset{on line, 
        boxsep=3pt, left=0pt,right=0pt,top=0pt,bottom=0pt,
        boxrule=0mm,
        colframe=white
        }

\definecolor{orange}{rgb}{1,0.5,0}
\definecolor{mdred}{rgb}{0.7,0,0}
\definecolor{mdgreen}{rgb}{0.05,0.6,0.05}
\definecolor{mdblue}{rgb}{0,0,0.7}
\definecolor{dkblue}{rgb}{0,0,0.5}
\definecolor{dkgray}{rgb}{0.3,0.3,0.3}
\definecolor{slate}{rgb}{0.25,0.25,0.4}
\definecolor{gray}{rgb}{0.5,0.5,0.5}
\definecolor{ltgray}{rgb}{0.7,0.7,0.7}
\definecolor{purple}{rgb}{0.7,0,1.0}
\definecolor{lavender}{rgb}{0.65,0.55,1.0}
\definecolor{green1}{rgb}{0.0, 0.6, 0.0}
\definecolor{green2}{rgb}{0.1, 0.7, 0.5}
\definecolor{green3}{rgb}{0.2, 0.8, 0.7}

\newcommand{\surrogatecb}[1]{\tcbox[colback=orange!20]{#1}}
\newcommand{\relaxcb}[1]{\tcbox[colback=blue!20]{#1}}
\newcommand{\samplingcb}[1]{\tcbox[colback=green3!40]{#1}}
\newcommand{\sfecb}[1]{\tcbox[colback=green2!30]{#1}}
\newcommand{\reparamcb}[1]{\tcbox[colback=green1!20]{#1}}

\newcommand{\surrogatecitet}[1]{\surrogatecb{\citet{#1}}}
\newcommand{\relaxcitet}[1]{\relaxcb{\citet{#1}}}

\newcommand{\sfecitet}[1]{\sfecb{\citet{#1}}}
\newcommand{\reparamcitet}[1]{\reparamcb{\citet{#1}}}

\newcommand{\papercomment}[3]{\ensuretext{\textcolor{#3}{[#1 #2]}}}
\renewcommand{\papercomment}[3]{}  %

\title{Learning with Latent Structures in Natural Language Processing: A Survey}
\author{\Large Zhaofeng Wu \normalsize\\\\University of Washington\\Allen Institute for Artificial Intelligence\\\texttt{zfw7@cs.washington.edu}}
\date{\vspace{-1ex}}

\begin{document}
\maketitle

\begin{abstract}
	While end-to-end learning with fully differentiable models has enabled tremendous success in natural language process (NLP) and machine learning, there have been significant recent interests in learning with latent discrete structures to incorporate better inductive biases for improved end-task performance and better interpretability. This paradigm, however, is not straightforwardly amenable to the mainstream gradient-based optimization methods. This work surveys three main families of methods to learn such models: surrogate gradients, continuous relaxation, and marginal likelihood maximization via sampling. We conclude with a review of applications of these methods and an inspection of the learned latent structure that they induce.\footnote{This work is inspired by \citet{martins-etal-2019-latent}.}
\end{abstract}

\begin{table}[t]
\centering
\begin{tabular}{@{} c @{\hspace{1pt}} | @{\hspace{1pt}} c @{\hspace{1pt}} | @{\hspace{3pt}} c @{\hspace{3pt}} c @{\hspace{3pt}} c @{\hspace{3pt}} c @{}}
\toprule

\multirow{2}{*}{\textbf{Family}} & \multirow{2}{*}{\textbf{Method}} & \textbf{Structure} & \textbf{Discrete} & \textbf{Inference} \\
&& \textbf{or Not} & \textbf{Output} & \textbf{Algo.} \\

\midrule

\multirow{2}{*}{Surrogate} & \surrogatecb{STE}~\citep{hinton2012neural} & Both & \cmark & MAP \\
& \surrogatecb{SPIGOT}~\citep{peng-etal-2018-backpropagating} & \cmark & \cmark & MAP \\

\midrule

\multirow{4}{*}{Relaxation} & \relaxcb{$\softmax$} & \xmark & \xmark \\
& \relaxcb{$\operatorname{sparsemax}$}~\citep{10.5555/3045390.3045561}\relaxcb{+}$^\dagger$ & \xmark & \cmark \\
& \relaxcb{Part-Marginalization} & \cmark & \xmark & Marg. \\
& \relaxcb{SparseMAP}~\citep{sparsemap} & \cmark & \cmark & MAP \\

\midrule

\multirow{6}{*}{Sampling} & \sfecb{Score Function Estimator}~\citep{10.1007/BF00992696} & Both & \cmark & Sampling \\
& \reparamcb{Rectified Distributions}~\citep{louizos2018learning} & \xmark & \phantom{$^{\ddagger}$}\xmark$^{\ddagger}$ \\
& \reparamcb{Gumbel-Max}~\citep{GumbelEmilJulius1954Stoe} & \xmark & \cmark \\
& \reparamcb{Gumbel-softmax}~\citep{gumbel-softmax,maddison2017concrete} & \xmark & \xmark \\
& \reparamcb{Perturb-and-Parse}~\citep{corro2018differentiable} & \cmark & \xmark & MAP \\
& \reparamcb{Direct Loss Minimization}~\citep{mcallester2010direct} & \xmark & \cmark \\

\bottomrule

\end{tabular}
\caption{\label{tab:summary} A high-level summary of the family and the methods reviewed in this survey. $^\dagger\operatorname{sparsemax}$+ refers to $\operatorname{sparsemax}$ and other related functions (see \Cref{sec:relaxation}). We list if the methods are primarily used for inducing structured or unstructured objects, or both. Methods marked with \xmark \ may still be employed for structured induction when the parser decisions are local, such as transition-based parsers (see \Cref{sec:examples}). We also mark if the methods yield discrete output, which can be desirable for, for example, interpretability, and sometimes enables specialized algorithms (e.g., \citealp{hu-etal-2021-r2d2}). $^\ddagger$Some rectified distributions still have a non-zero probability to output continuous solutions, though they may be discretizable. Finally, we indicate the structured inference algorithms required for all methods that are used for inducing structures: MAP, marginalization, or sampling. Depending on the problem, there may not exist an efficient version of a particular inference algorithm, affecting the choice of methods.}
\end{table}

\section{Introduction} \label{sec:intro}
	
With recent advances in deep learning, end-to-end differentiable modules such as LSTM~\citep{lstm} and Transformer~\citep{transformer}, coupled with gradient-based optimization, have become the mainstream paradigm in NLP and have helped achieve state-of-the-art results in NLP without much task-specific, language-specific, or linguistically-inspired design~\citep[\emph{inter alia}]{peters-etal-2018-deep,devlin-etal-2019-bert,2020t5}.

On the other hand, human languages possess inherent structures, such as syntactic, semantic, and discourse structures~\citep{chomsky1965aspects,heim1998semantics}. Consequently, NLP systems have historically leveraged such structures to improve their quality. This was traditionally done with pipelined systems, where a separately trained parser is first used to decode an intermediate structure, which is then used as input for downstream tasks~\citep[\emph{inter alia}]{sennrich-haddow-2016-linguistic,eriguchi-etal-2016-tree,chen-etal-2017-enhanced}. This paradigm, however, faces a few challenges. First, its disjoint modules tend to suffer from cascading errors. Second, such parsers are usually trained in a supervised fashion which, in order to ensure coverage for robustness, can require large amounts of annotated data. Such annotations often require domain expertise and, despite recent efforts such as the Universal Dependencies project~\citep{ud}, remain difficult and expensive to obtain, especially for low-resource languages and domains. Even when sufficient data exist to train a high-quality formalism-specific parser, this specific structural formalism may not be the optimal one for the downstream task, and it is sometimes advantageous to allow the model to discover the best structure targeting the downstream task~\citep[\emph{inter alia}]{kim2017structured,choi2018learning,maillard_clark_yogatama_2019}.

Therefore, previous studies have attempted to jointly train the parser and the downstream predictor and allow the parser to induce the best task-specific structure without the need for structural supervision, a paradigm which we review in this survey.
Such systems, however, suffer from one major difficulty. To induce an intermediate \emph{discrete} structure, the $\argmax$ function is usually used to make decisions as a layer in the parser. However, the gradient of $\argmax$ is either zero (almost everywhere) or undefined, so it does not work well with gradient-based optimization methods, the dominant optimization scheme in NLP and deep learning.

In this work, we survey three main families of approaches that tackle the non-differentiability of $\argmax$. The first family uses biased estimators with some proxy to the true gradient. The second circumvents discreteness using continuous relaxation. The last family casts the optimization problem as marginal likelihood maximization and uses sampling to approximate the gradient. We summarize these methods and several of their key properties in \Cref{tab:summary}. We hope this work could catalyze future studies in augmenting recent NLP models with structural inductive biases.

\section{Background: Structure Prediction} \label{sec:structure-prediction}

\begin{figure*}
	\centering
	\begin{subfigure}{.41\linewidth}
		\includegraphics[height=1.4in]{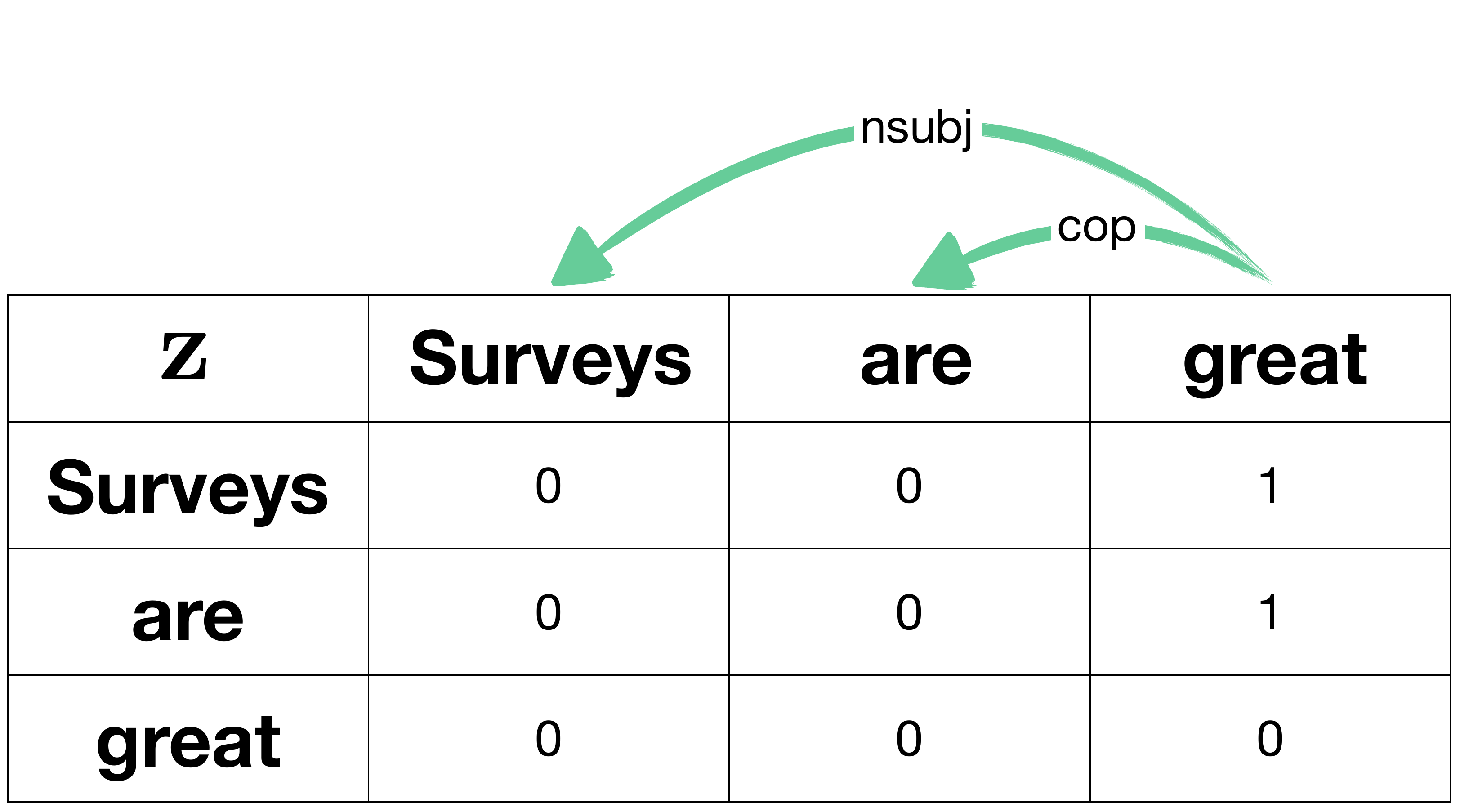}
		\caption{Dependency structure. We show edge labels, though they are ignored in $\vz$.}
		\label{fig:dependency-example}
	\end{subfigure}
	\hspace{.05\linewidth}
	\begin{subfigure}{.52\linewidth}
		\includegraphics[height=1.4in]{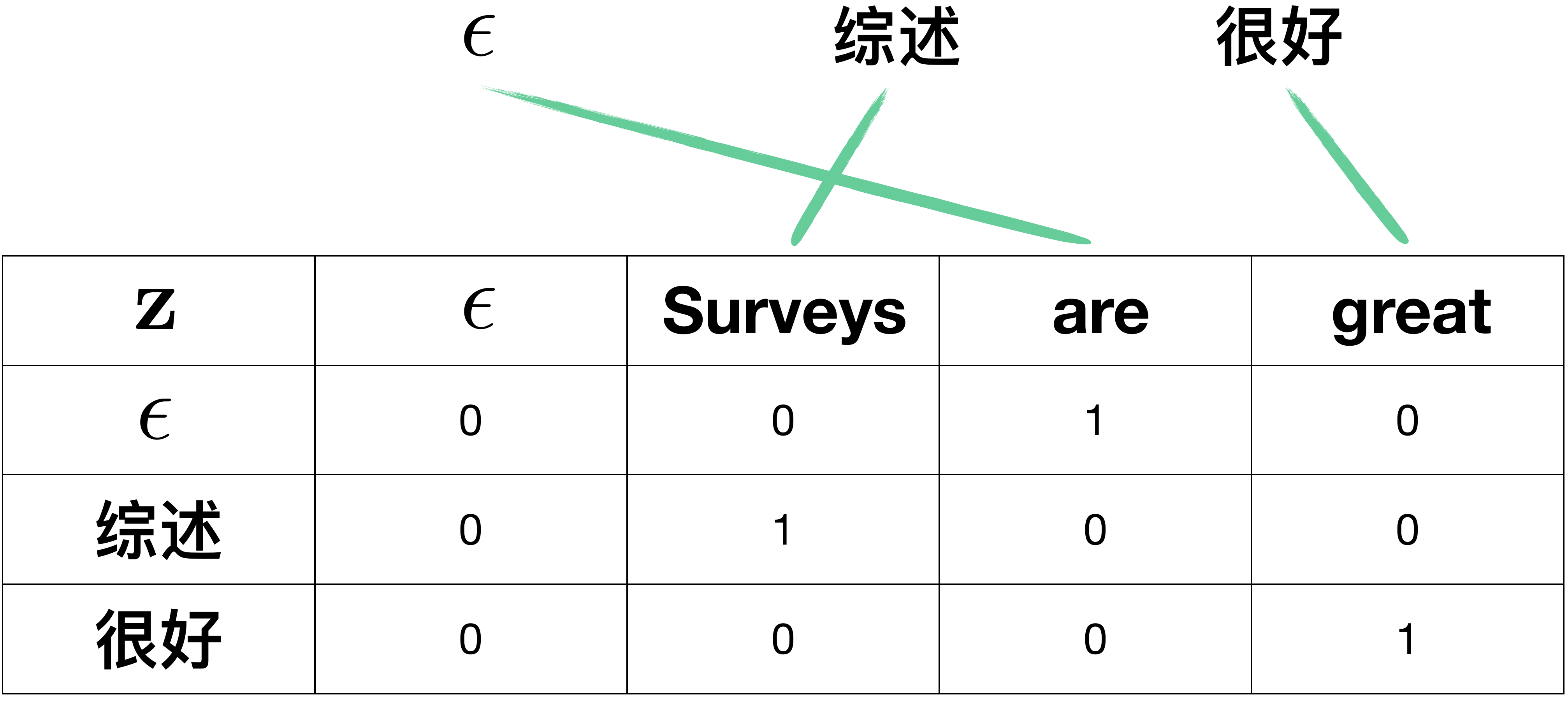}
		\caption{Matching structure. A machine translation example is shown here, though it can be seen in many other tasks.}
		\label{fig:matching-example}
	\end{subfigure}
	\caption{
		Examples of structures as well as, after flattening the binary entries in the tables, their associated $\vz$.
	}
	\label{fig:structure-examples}
\end{figure*}

The task of structure prediction aims to automatically extract structures that, in NLP, are often rooted in linguistic theories~\citep{smith:2011:synthesis}. In this survey, we loosely define ``structure'' and consider any collection of inter-dependent variables as a structure. Abstracting away from specific formalisms, we generically denote any structure with $\vz\in\cZ$ where $\cZ$ is the set of all possible structures and encodes any structural constraints enforced by a formalism.\footnote{$\cZ$ usually depends on $x$, but we drop this dependency notationally for clarity.} As an example, for many bilexical dependency formalisms, syntactic~\citep{ud} or semantic~\citep{ivanova-etal-2012-contrastive}, $\vz$ can be viewed as a collection of binary parts $\vz=[z_1,\cdots,z_{n}]\in\{0,1\}^{n^2}$ where each entry denotes the existence of an edge between a pair of words in a sentence. Here, the structural constraint encoded by $\cZ$ could require, for instance, the arborescence structure for syntactic dependencies. As another example, in matching problems, the words in a source sentence with length $m$ are matched with the words in a target sentence with length $n$. Accounting for non-matched words by matching them with a special $\eps$ symbol, $\vz$ can again be viewed as consisting of binary parts, with length $(m+1)(n+1)$, where each entry denotes if two words are matched. These two types of structures are visualized in \Cref{fig:structure-examples}, and more examples can be found in \Cref{sec:examples}.

A parser finds the highest-scoring structure $\hat{\vz}=\argmax_{\vz\in\cZ} S(\vz \mid x; \theta)$ for some input $x$ with some scoring function $S$ parameterized by $\theta$. We use $\pi(\vz \mid x; \theta)$ to denote a \emph{normalized} scoring function which yields a probability distribution over all $\vz\in\cZ$.
For example, for dependency parsing, a graph-based parser is commonly used to produce a score for every part (for the example in \Cref{fig:dependency-example}, this would be a score for every cell), and the scores that correspond to the activated parts in a structure $\vz$ are summed. Specifically, we can write the scoring function as $S(\vz \mid x; \theta)= \vz^\top f(x;\theta) = \vz^\top\vs$ where $\vs:=f(x;\theta)$
denotes the per-part score produced by some model $f$.

In this survey, for simplicity, we default to settings and terminologies for graph-based dependency parsing.

\section{Latent Structure Prediction} \label{sec:latent-structure-prediction}

As argued in Section~\ref{sec:intro}, it is often desirable to model discrete structures latently as an intermediate representation in a system. Specifically, we still have $\vs=f(x;\theta), \hat{\vz}=\argmax_{\vz\in\cZ}\vz^\top\vs$ as the best structure. We then feed $\hat{\vz}$ to some downstream task-specific predictor $\hat{y}=g(\hat{\vz};\phi)$ parameterized by $\phi$ and incur some loss $\ell=L(\hat{y},y)$.
For some applications, the original input $x$ is also used as an input to $g$ in addition to the intermediate structure $\hat{\vz}$, but we omit it for clarity.

\begin{figure*}
	\centering
	\begin{subfigure}{.49\linewidth}
		\includegraphics{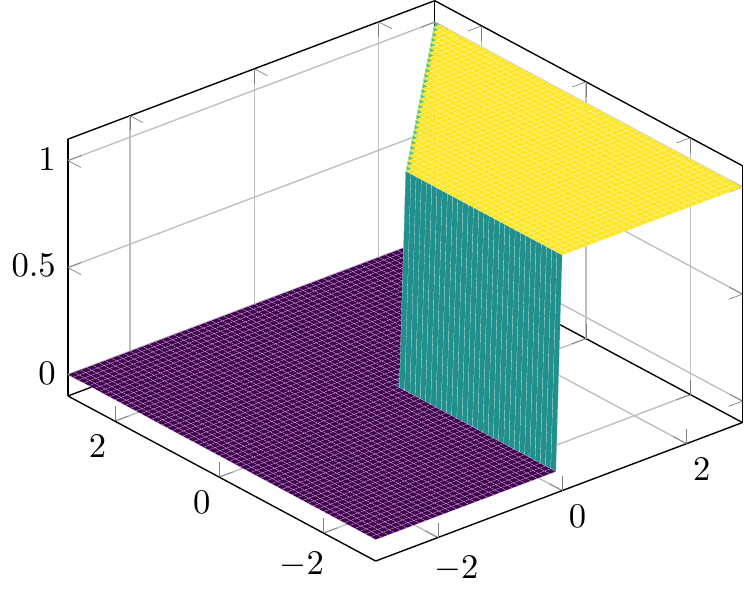}
		\caption{$\argmax$. The gradient is either zero (almost everywhere) or undefined.}
		\label{fig:argmax}
	\end{subfigure}
	\begin{subfigure}{.49\linewidth}
		\includegraphics{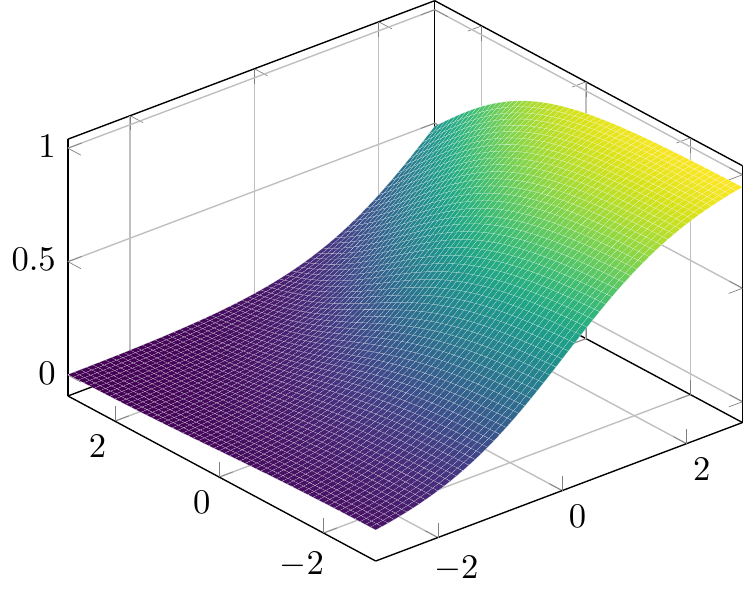}
		\caption{$\softmax$. The gradient is well-defined everywhere.}
		\label{fig:softmax}
	\end{subfigure}
	\caption{
		Plot of the \textbf{(a)} $\argmax$ and \textbf{(b)} $\softmax$ functions.
		Figure is taken from \citet{corro-titov-2019-learning}.
	}
	\label{fig:argmax-softmax}
\end{figure*}

$g$ can be seen as a structure encoder. With linguistic structures, for example, it usually takes the form of graph neural networks which encodes the structure $\hat{\vz}$. TreeLSTMs~\citep{tai-etal-2015-improved} have been widely used~\citep{maillard_clark_yogatama_2019,choi2018learning}, including variants that use different composition functions~\citep{hu-etal-2021-r2d2}. With the advent of graph convolutional networks~\citep{kipf2017semi}, more works have based themselves on this architecture and its variants~\citep{corro-titov-2019-learning,wu2021infusing}.

In practice, $g$ can be pretrained before serving as a frozen component in the final system~\citep{wu-etal-2019-wtmed,marcheggiani-titov-2017-encoding,wu2021infusing}, or jointly trained with $f$ with separate loss signals using multi-task learning~\citep{bowman-etal-2016-fast,hashimoto-etal-2017-joint}. Nevertheless, driven by the desiderata outlined in \Cref{sec:intro}, in this survey, we only consider the case where $f$ and $g$ are jointly trained using only the downstream loss.

The challenge associated with this paradigm lies in the gradient (or Jacobian, in the case of structured output) of $\argmax$, which is either zero (almost everywhere) or undefined (\cref{fig:argmax}). We hence do not have a meaningful gradient with respect to
$\theta$ for optimization.
Below, we introduce three main families of methods that tackle this difficulty.

We note that the non-differentiability of $\argmax$ poses an issue not only for latent structure prediction, but also for non-structural discrete latent decisions, such as in the context of quantization~\citep{quantization}. Therefore, for most methodological families that we review below, we start with the unstructured case. Nevertheless, structure prediction comes with additional difficulties. For example, the feasible set is usually exponential in the number of structural parts, effectively ruling out enumeration, and it also often comes with structural constraints. These methods, therefore, sometimes cannot be directly applied to the structural setting and need modifications to overcome these challenges.

\section{\surrogatecb{Surrogate Gradients}} \label{sec:surrogate-gradients}

One family of approaches directly optimize $L(g(\hat{\vz};\phi),y)$ by introducing some gradient-like quantity for backpropagation. Backpropagating through the model described in \Cref{sec:latent-structure-prediction} yields a well-defined $\nabla_{\hat{\vz}}\ell$. However, since $\argmax$ is used to obtain $\hat{\vz}$ from $\vs$, $\nabla_{\vs}\ell=0$, and hence $\nabla_\theta\ell=0$ and the parameters of the parser will not be updated.

In straight-through estimator (STE; \citealp{hinton2012neural,bengio2013estimating}), $\nabla_{\vs}\ell$ is manually overridden. For example, in the identity STE:
\begin{align} \label{eq:ste}
	\nabla_{\vs}\ell=\nabla_{\hat{\vz}}\ell
\end{align}
as if the identity activation $\hat{\vz}=\vs$ were performed in the forward pass (cf. the actual $\hat{\vz}=\argmax_{\vz\in\cZ}\vz^\top\vs$). It effectively ``skips'' the $\argmax$ operation during the backward pass, hence the name ``straight-through.'' Alternatives exist that override the gradient with that of other activation functions including $\operatorname{sigmoid}$~\citep{bengio2013estimating} and variants of the $\tanh$~\citep{hubara2016bnn} and $\operatorname{ReLU}$~\citep{cai17hwgq} functions. This procedure results in a biased estimator for the true gradient. Nevertheless, despite this biased nature and its simplicity, it empirically works well~\citep{bengio2013estimating,chung2017hierarchical}. Much research in understanding and justifying why such a mismatch between the objective and the learning algorithm remains beneficial is still underway. For example, \citet{yin2018understanding} showed that the gradients provided by properly chosen STE variants correlate to the descent direction that minimizes the population loss.

Being mindful of the structural constraints for STE can help improve the learning process. To see this, we first examine an alternative formulation of STE as proposed by \citet{mihaylova-etal-2020-understanding}.
Consider the hypothetical case where the optimal structure $\vz^*$ is accessible. It can be used to incur some \emph{intermediate} loss, such as the structured perceptron loss $L^{\text{structure}}(\hat{\vz},\vz^*)=\vs^\top\hat{\vz}-\vs^\top\vz^*$.\footnote{Note that the intermediate $L^{\text{structure}}$ is different from the downstream loss function $L$.} This would allow a nonzero gradient with respect to $\vs$: $\nabla_{\vs}\ell=\hat{\vz}-\vz^*$. In practice, however, such an optimal structure may not always be accessible. Nevertheless, an approximation to the optimal structure can be induced using the downstream loss $\ell$. One can view $\hat{\vz}$ as an ``updatable'' quantity and use $\tilde{\vz}=\hat{\vz}-\nabla_{\hat{\vz}}\ell\approx\vz^*$ for this approximation. Effectively, this procedure performs a one-step gradient descent to induce an approximation of the optimal structure. \citet{mihaylova-etal-2020-understanding} noted that this formulation recovers the identity STE in \Cref{eq:ste}:
\begin{align}
	\nabla_{\vs}\ell=\hat{\vz}-\vz^*\approx\hat{\vz}-\tilde{\vz}=\hat{\vz}-\left(\hat{\vz}-\nabla_{\hat{\vz}}\ell\right)=\nabla_{\hat{\vz}}\ell
\end{align}

\begin{figure}[t]
	\centering
	\includegraphics[width=0.5\textwidth]{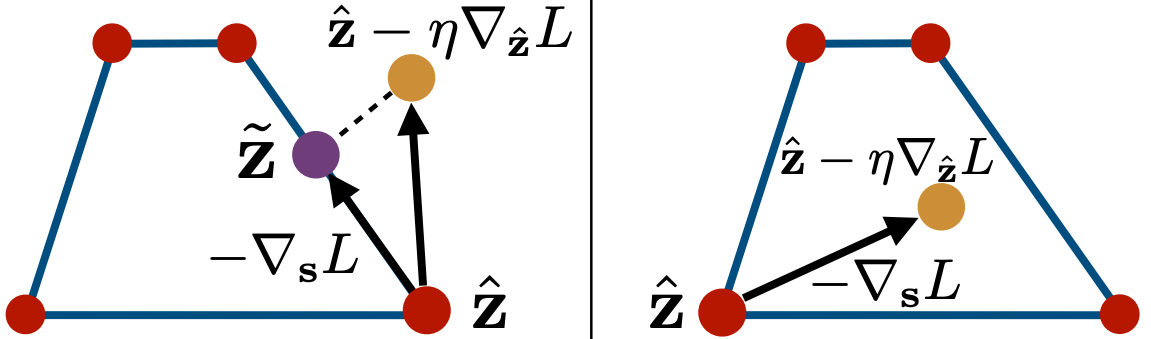}
	\caption{Left: When the gradient update yields a point outside of the convex polytope $\cP$, we perform a projection back to $\cP$. Right: When the gradient update yields a point inside $\cP$, no projection is needed. Figure is taken from \citet{peng-etal-2018-backpropagating}.}
	\label{fig:spigot}
\end{figure}

\citet{peng-etal-2018-backpropagating} noted that this approximation $\tilde{\vz}$ can be outside of the feasible set $\cZ$ or its convex hull $\cP = \operatorname{conv}(\cZ)$. So $\tilde{\vz}$ can be an invalid, let alone optimal, structure. They proposed the method SPIGOT that projects $\tilde{\vz}$ back to $\cP$ when it is outside of $\cP$:
\begin{align}
	\tilde{\vz}^{\text{SPIGOT}}=\operatorname{proj}_\cP(\hat{\vz}-\nabla_{\hat{\vz}}\ell)\approx\vz^*
\end{align}
Then the gradient with respect to $\vs$ can be similarly calculated: $\nabla_{\vs}\ell=\hat{\vz}-\vz^*\approx\hat{\vz}-\tilde{\vz}^{\text{SPIGOT}}$.
Compared to the one-step gradient descent in the case of STE, SPIGOT can be viewed as performing one step of projected gradient descent. This process is illustrated in Figure~\ref{fig:spigot}.

Inspired by SPIGOT, \citet{mihaylova-etal-2020-understanding} explored more numbers of (projected) gradient update steps, an alternative loss based on the CRF loss~\citep{10.5555/645530.655813} in the place of the structured perceptron loss, and exponentiated gradient updates~\citep{KIVINEN19971}. Nevertheless, they did not observe significant improvement in the structured case. They hypothesized that an optimal intermediate approximation may not overall help learning.
\section{\relaxcb{Continuous Differentiable Relaxation}} \label{sec:relaxation}
Another family of approaches relax the nondifferentiable $\argmax$ function with some continuous and differentiable function. In general, rather than taking a single $\hat{\vz}=\argmax_{\vz\in\cZ}S(\vz \mid x;\theta)$, they consider $\hat{\vz}=\E_{\vz\sim\pi(\vz \mid x;\theta)}[\vz]$ (or $\hat{z}_i=\E_{\vz\sim\pi(\vz \mid x;\theta)}[z_i]$ for a specific part $z_i$ in the structure), where $\pi(\vz \mid x;\theta)$ represents the normalized scores.
Writing this out:
\begin{align} \label{eq:relaxation}
    \hat{\vz}=\E_{\vz\sim\pi(\vz \mid x;\theta)}[\vz] = \sum_{\vz \in \cZ} \pi(\vz \mid x;\theta) \vz
\end{align}
We can see that, in \Cref{eq:relaxation}, rather than only considering one best $\vz$, it encodes uncertainty in the latent space by taking a weighted average of all $\vz$.
The attention mechanism is a common example of this type of softening. In the hard case, a model aligns to the token with the highest compatibility score. This, however, is non-differentiable and restricts the focus to only one token. Hence, recent works obtain a distribution among all candidate tokens, a paradigm that has achieved tremendous success~\citep[\emph{inter alia}]{attention,transformer}.

There can be many different choices of $\pi(\cdot)$.
A straightforward and widely-used choice uses the $\softmax$ function. Consider first some categorical and unstructured latent variable where $\cZ$ can be seen as the set of classes represented as one-hot vectors. Then:
\begin{align} \label{eq:onehot-softmax}
	\pi(\vz \mid x;\theta)=\vz^\top\softmax(\vs)
\end{align}

Nevertheless, one downside of $\softmax$ is that it is soft. Even when the temperature parameter~\citep{temperature} is sometimes used to induce sparse probabilities~\citep{gumbel-softmax,zhou-neubig-2017-multi,maillard_clark_yogatama_2019}, it never yields completely discrete 0-1 solutions.
On the other hand, discreteness is sometimes desirable for, for example, interpretability, so many works have sought sparser alternatives to $\softmax$ that are still differentiable.

These alternatives are inspired by the equivalence of $\softmax$ and $\argmax$ coupled with a Shannon entropy term as penalty:\footnote{For a proof, see the appendix of \citet{niculae-etal-2018-towards}.}
\begin{align}
	\softmax(\vs)=\argmax_{\vz\in\Delta^{\abs{\cZ}-1}}\vz^\top\vs+H(\vz)
\end{align}
Substituting the entropy term $H(\vz)$ with other forms of penalty results in functions including $\operatorname{sparsemax}$~\citep{10.5555/3045390.3045561}, $\alpha$-$\operatorname{entmax}$~\citep{peters-etal-2019-sparse}, $\operatorname{fusedmax}$~\citep{10.5555/3294996.3295093}, etc. As derived in \citet{10.5555/3294996.3295093}:
\begin{align}
	\operatorname{sparsemax}(\vs) &= \argmax_{\vz\in\Delta^{\abs{\cZ}-1}}\vz^\top\vs-\frac{1}{2}\norm{\vz}_2^2 \\
	\alpha\operatorname{-entmax}(\vs) &= \argmax_{\vz\in\Delta^{\abs{\cZ}-1}}\vz^\top\vs-\frac{1}{\alpha(\alpha-1)}\norm{\vz}_\alpha^\alpha \\
	\operatorname{fusedmax}(\vs) &= \argmax_{\vz\in\Delta^{\abs{\cZ}-1}}\vz^\top\vs-\frac{1}{2}\norm{\vz}_2^2-\lambda\sum_i \abs{z_i-z_{i-1}}
\end{align}

\begin{figure}[t]
	\centering
	\includegraphics[width=0.98\textwidth]{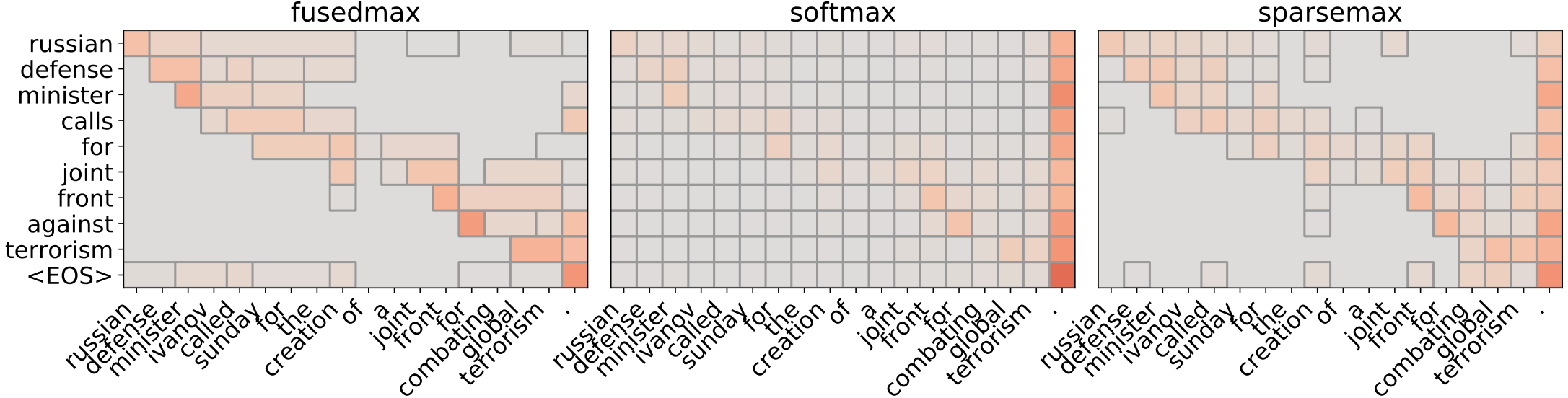}
	\caption{The attention distribution as normalized by three different functions. $\operatorname{fusedmax}$ and $\operatorname{sparsemax}$ produce sparser patterns than $\softmax$ while the attention distribution $\operatorname{fusedmax}$ has more contiguous chunks. Figure is taken from \citet{10.5555/3294996.3295093}.}
	\label{fig:relaxation-functions}
\end{figure}

\begin{figure}[t]
	\centering
	\includegraphics[width=0.6\textwidth]{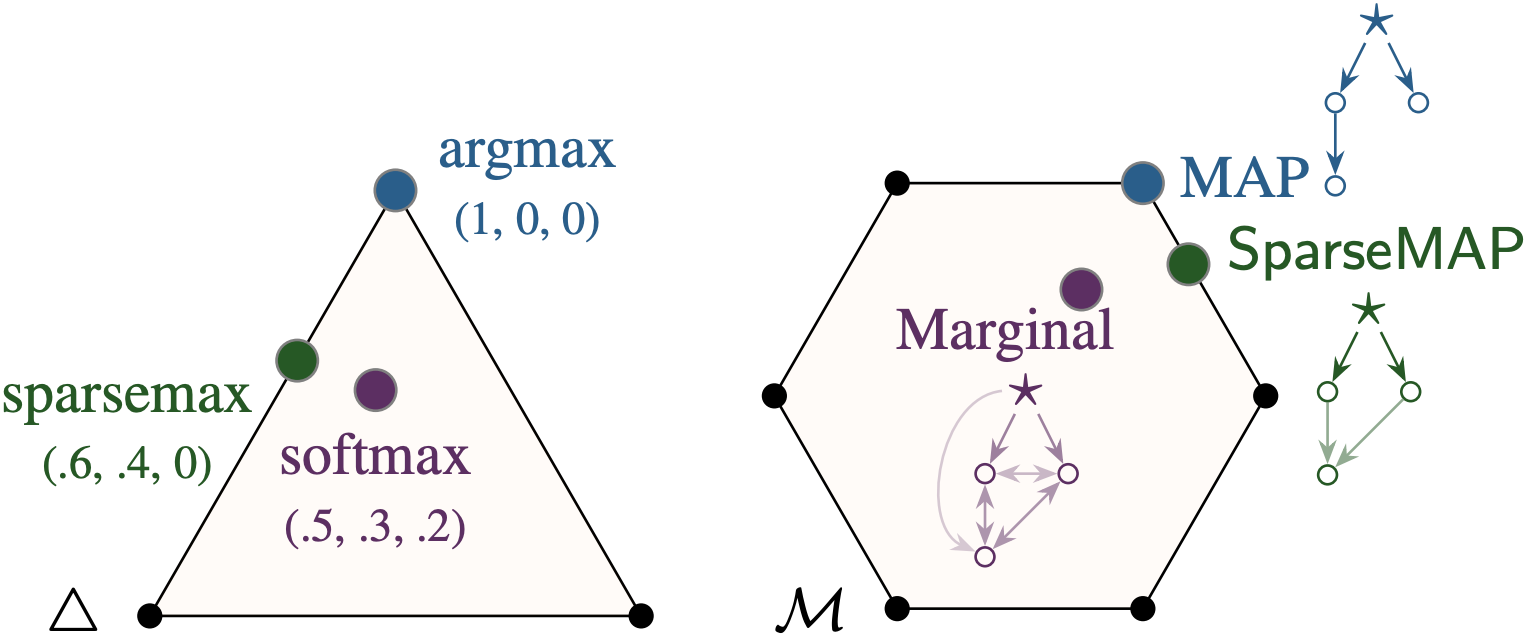}
	\caption{Illustration of the relationship between $\argmax$, $\softmax$, and $\operatorname{sparsemax}$ in the probability simplex (left), as well as their structured counterpart in the convex hull of all possible structures (right). Figure is taken from \citet{sparsemap}.}
	\label{fig:sparsemax-sparsemap}
\end{figure}

$\alpha$-$\operatorname{entmax}$, of which $\operatorname{sparsemax}$ is a special case, penalizes the norm of the structure, effectively encouraging sparser structures. $\operatorname{fusedmax}$ builds on $\operatorname{sparsemax}$ while including an additional penalty term that discourages changes in value between consecutive parts. See Figure~\ref{fig:relaxation-functions} for a visualization. \Cref{fig:sparsemax-sparsemap} (left) illustrates the relationship between $\argmax$, $\softmax$, and this family of approaches using $\operatorname{sparsemax}$ as an example.

Of course, these $\argmax$-based formulations cannot be na\"{i}vely applied since, again, $\argmax$ is not amenable to gradient-based optimization. Nevertheless, non-zero Jacobians specific to each of these functions have been derived. We refer readers to the original papers for details.

When $\vz$ is a structure, $\E_{\vz\sim\pi(\vz \mid x;\theta)}[\vz]$ can be intractable to obtain, and people have instead considered the part-marginal $\E_{\vz\sim\pi(\vz \mid x;\theta)}[z_i]$ by marginalizing over exponentially many $\vz$. Specifically:
\begin{align} \label{eq:marg}
	\pi(Z_i=z_i \mid x;\theta) = \E_{\vz\sim\pi(\vz \mid x;\theta)}[z_i] = \sum_{\vz\in\cZ} \pi(\vz \mid x;\theta) \mathbb{I}([\vz]_i=z_i)
\end{align}
where $\mathbb{I}(\cdot)$ is the indicator function. Of course, directly calculating the right-hand side of \Cref{eq:marg} is still in most cases intractable. Nevertheless, for specific cases, efficient algorithms exist, typically variants of the sum-product algorithm, such as the forward-backward algorithm~\citep{18626} and the inside-outside algorithm~\citep{doi:10.1121/1.2017061}.

Analogous to $\operatorname{sparsemax}$ and related functions in the unstructured case, sparser alternatives to \Cref{eq:marg} have been pursued. While the right-hand side of \Cref{eq:marg} cannot be directly evaluated, with a sparse distribution overs structures $\pi(\cdot)$, we can only enumerate over structures with a nonzero probability which allows tractable inference. This provides an alternative to using sum-product algorithms which in some cases are unavailable and, even when available, only provide local features (e.g., the probability of an arc between a pair of words), incapable of modeling more global structures which can be beneficial~\citep{corro-titov-2019-learning}. \citet{sparsemap,niculae-etal-2018-towards} developed the SparseMAP inference strategy which produces a combination of only a small number of structures. \Cref{fig:sparsemax-sparsemap} illustrates its correspondence with the unstructured setting. In addition to outperforming the $\softmax$ baseline on a variety of tasks, this sparse distribution over structures allows for improved interpretability: they discovered that it represents legitimate linguistic ambiguity. \citet{lp-sparsemap} proposed LP-SparseMAP that builds on SparseMAP while allowing approximate inference.
\section{\samplingcb{Marginal Likelihood Maximization via Sampling}}

We now reconsider our objective probabilistically as a conditional marginal likelihood model. We want to maximize the likelihood of the label $y$ given the input $x$, $p(y\mid x)$. Introducing and marginalizing over the latent variable $\vz$, this can be written as
\begin{align} \label{eq:sampling}
	\begin{split}
		p(y\mid x; \theta,\phi) &= \sum_{\vz\in\cZ} p(y\mid x,\vz;\phi) p(\vz\mid x;\theta) \\
		&= \E_{\vz\sim p(\vz \mid x;\theta)}[p(y\mid x,\vz;\phi)]
	\end{split}
\end{align}

This perspective in fact relates to using continuous relaxation where we considered $p(y\mid x,\E_{\vz\sim\pi(\vz \mid x;\theta)}[\vz];\phi)$. When the function that models the outer probability distribution is convex, though it often is not as parameterized by neural networks, by Jensen's inequality, the continuous relaxation perspective is a lower bound of this quantity. Hence, optimizing \Cref{eq:sampling} provides an alternative to continuous relaxation.

Nevertheless, the challenge of optimizing \Cref{eq:sampling}, or back to our formulation, $\E_{\vz\sim\pi(\vz \mid x;\theta)}[L(g(\vz;\phi),y)]$, with respect to $\theta$ lies in the fact that the expectation is taken over a distribution parameterized by $\theta$. This makes $\nabla_\theta \E_{\vz\sim\pi(\vz \mid x;\theta)}[L(g(\vz;\phi),y)]$ difficult to calculate.
Sometimes this gradient can be directly evaluated with special structures and algorithms (e.g., \citet{wiseman-etal-2018-learning}). In this survey, we focus on the more general setting where this might be intractable and review two methods that rewrite this quantity in forms that allow its approximation via sampling: score function estimator and reparameterization.\footnote{Technically, this formulation does not apply directly in an autoencoding setup where $y=x$, because we cannot use two separate networks to model $p(x\mid\vz)$ and $p(\vz\mid x)$. Directly maximizing the probability is often intractable (\citet{ammar2014conditional} being a notable exception), and requires optimizing the evidence lower bound, or ELBO. We omit the technical details for brevity, though most methods below also apply to ELBO optimization. See \citet{rush-etal-2018-deep} for a review. \label{fn:vae}}

\subsection{\sfecb{Score Function Estimator}} \label{sec:sfe}

One class of such methods uses score function estimators, also known as the \textsc{reinforce} algorithm~\citep{10.1007/BF00992696}, to compute this gradient. They are based on the identity
\begin{align}
	\nabla_{\theta}\pi(\vz \mid x;\theta)=\pi(\vz \mid x;\theta)\nabla_{\theta}\log\pi(\vz \mid x;\theta)
\end{align}
which leads to
\begin{align} \label{eq:sfe}
	\begin{split}
		\nabla_{\theta} \E_{\vz\sim\pi(\vz \mid x;\theta)}[L(g(\vz))] &= \nabla_{\theta} \sum_{\vz\in\cZ} L(g(\vz)) \pi(\vz \mid x;\theta) \\
		&= \sum_{\vz\in\cZ} L(g(\vz)) \nabla_{\theta} \pi(\vz \mid x;\theta) \\
		&= \sum_{\vz\in\cZ} L(g(\vz)) \pi(\vz \mid x;\theta)\nabla_{\theta}\log\pi(\vz \mid x;\theta) \\
		&= \E_{\vz\sim\pi(\vz \mid x;\theta)}[L(g(\vz))\nabla_{\theta}\log\pi(\vz \mid x;\theta)]
	\end{split}
\end{align}
This formulation allows approximating the gradient with Monte-Carlo samples.

The score function estimator can suffer from high variance due to its sampling nature considering the exponential number of possible trees and the lack of intermediate rewards. For example, \citet{nangia-bowman-2018-listops} showed that the trees induced by \citet{DBLP:conf/iclr/YogatamaBDGL17}, a method based on score function estimators, has a self-F1, which is a measure of variance, that is similar to random trees. Hence, variance reduction techniques are often employed in practice. For example, a control variate $b(x)$~\citep{10.5555/3042573.3042748} can be subtracted from $L$:
\begin{align}
	\nabla_{\theta} \E_{\vz\sim\pi(\vz \mid x;\theta)}[L(g(\vz))] =
	\E_{\vz\sim\pi(\vz \mid x;\theta)}[(L(g(\vz))-b(x))\nabla_{\theta}\pi(\vz \mid x;\theta)]
\end{align}
	
\subsection{\reparamcb{Reparameterization}} \label{sec:reparam}

An alternative to score function estimators is the reparameterization trick~\citep{vae}. Suppose we can sample from $\pi(\cdot\mid x;\theta)$ by first sampling a noise $\gamma$ from a simpler base distribution $\cB$ \emph{independent} of $\theta$ and then apply a transformation $h(\gamma;\theta)$. Then:
\begin{align} \label{eq:reparam}
    \begin{split}
        \nabla_{\theta} \E_{\vz\sim\pi(\vz \mid x;\theta)}[L(g(\vz))] &= \nabla_{\theta} \sum_{\vz\in\cZ} L(g(\vz)) \pi(\vz \mid x;\theta) \\
        &= \nabla_{\theta} \sum_{\gamma\in\operatorname{supp}(\cB)} L(g(h(\gamma; \theta))) \cB(\gamma) \\
        &= \sum_{\gamma\in\operatorname{supp}(\cB)} \cB(\gamma) \nabla_{\theta} L(g(h(\gamma; \theta))) \\
        &= \E_{\gamma\sim\cB}[\nabla_{\theta} L(g(h(\gamma; \theta)))]
    \end{split}
\end{align}
This allows, once again, Monte-Carlo methods to approximate the original gradient.

We can compare the score function estimator and reparameterization-based methods by contrasting the final line in Equations~(\ref{eq:sfe}) and (\ref{eq:reparam}). In reparameterization, we differentiate through the loss function, whereas in score function estimator, we only consider the gradient of the log probability of the structure and use the loss, or ``reward'' in reinforcement learning terminologies, as a black-box to drive the optimization. The former hence is more informative and empirically yields lower variance~\citep{rush-etal-2018-deep}. Nevertheless, not all probability distributions are reparameterizable.

The reparameterization trick was originally introduced to train variational auto-encoders~\citep{vae}. Thanks to the linearity of the normal distribution, sampling $X\sim\mathcal{N}(\mu,\sigma^2)$ can be equivalently reparameterized as $X=h(\gamma; \theta)=\mu+\sigma \gamma$ where $\gamma\sim\mathcal{N}(0,1)$, allowing gradients to be taken with respect to $\mu$ and $\sigma$ so that they can be learned. Nevertheless, it can be more difficult to reparameterize a discrete variable. We review a few methods to do this below.

\begin{figure}[t]
    \centering
    \includegraphics[width=0.5\textwidth]{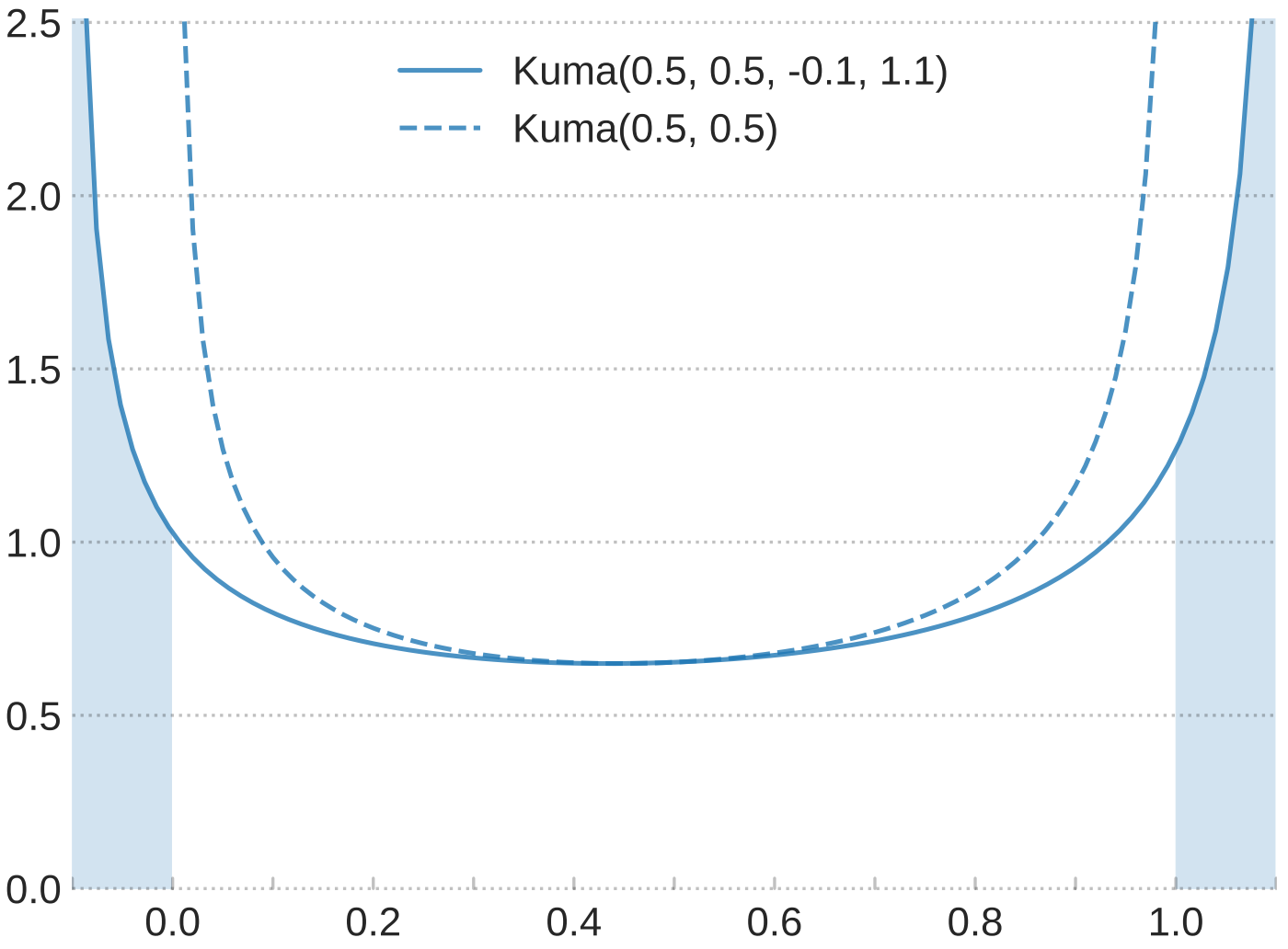}
    \caption{The Kumaraswamy distribution (dashed curve) and the HardKumaraswamy distribution (solid curve). The HardKumaraswamy distribution ``stretches'' the Kumaraswamy distribution and collapses mass beyond $[0,1]$. Figure is taken from \citet{bastings-etal-2019-interpretable}.}
    \label{fig:hardkuma}
\end{figure}

\subsubsection{Rectified Distributions}

Rectifying continuous distributions can ensure a non-trivial probability for discrete solutions. For example, \citet{bastings-etal-2019-interpretable} used a stretch-and-rectify procedure, originally proposed in \citet{louizos2018learning}, to design the continuous HardKumaraswamy distribution (solid curve in Figure~\ref{fig:hardkuma}) that is based on the Kumaraswamy distribution (dashed curve in Figure~\ref{fig:hardkuma}; \citealp{KUMARASWAMY198079}), which \emph{is} reparameterizable. Specifically, after a base noise is drawn $\gamma\sim\operatorname{unif}(0,1)$, it is fed through the inverse CDF of the Kumaraswamy distribution $k=F_K^{-1}(\gamma;a,b)\in(0, 1)$, stretched $t=l+(r-l)k$, and collapsed $z=\min(1,\max(0,t))$. Intuitively, while the original Kumaraswamy distribution does not have 0 and 1 in its support, we can stretch it from its original support $(0, 1)$ to $[l, r]$ and collapse the probability mass beyond 0 and 1 to the endpoints, allowing a non-trivial probability at these two points. This stretch-and-rectify procedure is visualized in \Cref{fig:hardkuma}, and we refer readers to \citet{bastings-etal-2019-interpretable} for more details. %

\subsubsection{Gumbel-Max}

The Gumbel-Max trick~\citep{GumbelEmilJulius1954Stoe,NIPS2014_309fee4e} can also be used to reparameterize the sampling of discrete variables. To sample from the \emph{unnormalized} scores $\vs$, the Gumbel-Max trick perturbs the scores with an additive noise $\gamma\sim\operatorname{Gumbel}(0,1)$, or equivalently, $U\sim\operatorname{unif}(0,1), \gamma=-\log(-\log(U))$. It can be shown that taking 
\begin{align} \label{eq:gumbel-max}
    \hat \vz = h(\gamma; \theta) = \argmax_{\vz\in\cZ} \vz^\top \vs+\gamma = \argmax_{\vz \in \cZ} S(\vz\mid x;\theta) + \gamma
\end{align}
where $\cZ$ consists of one-hot vectors, is equivalent to sampling from the $\softmax$-normalized $\vs$.\footnote{See the appendix of \citet{NIPS2014_309fee4e} for a proof.}

However, even with the Gumbel-Max trick, the non-differentiability still exists in \Cref{eq:gumbel-max}. Therefore, methods introduced in Section~\ref{sec:surrogate-gradients} or \ref{sec:relaxation} may still be required. With surrogate gradients, it becomes Gumbel STE, while with continuous relaxation, it yields Gumbel-softmax~\citep{gumbel-softmax,maddison2017concrete}. The two can be combined, where the forward pass uses a discrete $\argmax$ like in STE, while the $\softmax$-ed distribution is used in the backward pass (vs. the identity function in identity STE). This is sometimes called straight-through Gumbel-softmax~\citep{gumbel-softmax}.

In the structured case, it is intractable to perturb every possible $\vz\in\cZ$ separately. \citet{corro2018differentiable,corro-titov-2019-learning} proposed differentiable Perturb-and-Parse by perturbing scores of each part. Specifically:
\begin{align} \label{eq:perturb-and-parse}
    \hat{\vz} = \argmax_{\vz\in\cZ} \vz^\top (\vs+\gamma)
\end{align}
This $\argmax$ can be solved with max-product algorithms, though these algorithms usually have $\argmax$ operations within them that break backpropagation~\citep{mensch2018differntiable}. Therefore, \citet{corro2018differentiable,corro-titov-2019-learning} relaxed these operations with $\softmax$ that preserves the gradient flow. Nevertheless, unlike in the unstructured case, sampling using this method is not exact because it perturbs local parts rather than the global structure. This is apparent by contrasting Equations~(\ref{eq:gumbel-max}) and (\ref{eq:perturb-and-parse}).

\subsubsection{Direct Loss Minimization}

As an alternative to circumventing the non-differentiability introduced by Gumbel-Max using STE or Gumbel-softmax, \citet{lorberbom2019direct} proposed to adopt direct loss minimization, originally introduced by \citet{mcallester2010direct}. Direct loss minimization was developed to optimize arbitrary, not necessarily differentiable, loss functions under mild conditions. We state this formally in \Cref{thm:direct-loss-minimization}, taken from \citet{song2016training}, which is a generalization of the original theorem proposed in \citet{mcallester2010direct}:

\begin{theorem} \label{thm:direct-loss-minimization}
When given a finite set $\cY$, a scoring function $F(x,y,w)$, a data distribution, as well as a task-loss $L(y,\hat y)$, then, under some mild regularity conditions (see the supplementary material of \citet{song2016training} for details), the direct loss gradient has the following form:
\begin{equation}
    \nabla_w \E[L(y,y_w)] = \pm \lim_{\epsilon \rightarrow 0} \frac{1}{\epsilon} \E[\nabla_w F(x,y_{\text{direct}},w) - \nabla_w F(x,y_w,w)]\label{eq:DirectLossGradient}
\end{equation} 
with
\begin{align}
    y_w &= \argmax_{\hat y \in \cY} F(x,\hat y,w)\\
    y_{\text{direct}} &= \argmax_{\hat y \in \cY} F(x,\hat y,w) \pm \epsilon L(y,\hat y)
\end{align}
\end{theorem}
\begin{proof}
We refer readers to the supplementary material of \citet{song2016training}.
\end{proof}

In \Cref{thm:direct-loss-minimization}, in addition to executing the regular inference procedure to compute $y_w$, one also needs an additional \emph{loss-adjusted} inference for $y_\text{direct}$.

\citet{lorberbom2019direct} adapted \Cref{thm:direct-loss-minimization} to work with the Gumbel-Max reparameterization in variational auto-encoders (VAEs) that allows further expansion of \Cref{eq:reparam} by plugging in \Cref{eq:gumbel-max}. This is formalized in \Cref{thm:dlm-gumbel-max}, taken from their work and adapted to our notation:
\begin{theorem}
\label{thm:dlm-gumbel-max}
Assume that $S(\vz\mid x;\theta)$ is a smooth function of $\theta$. Then\footnote{Again ignoring the caveat in \Cref{fn:vae}.}
\begin{align} \label{eq:dlm-gumbel-max}
    \begin{split}
        \E_{\gamma\sim\cB}[\nabla_{\theta} L(g(h(\gamma; \theta)))] &= \E_{\gamma\sim\cB}[\nabla_{\theta} L(g(\hat \vz))] \\
        &= \nabla_\theta \E_{\gamma\sim\cB} [L(g(\hat\vz))] \\
        &= \lim_{\epsilon \rightarrow 0} \frac{1}{\epsilon} \E_{\gamma\sim\cB} [\nabla_\theta S(\hat\vz(\eps)\mid x;\theta) - \nabla_\theta S(\hat\vz\mid x;\theta)]
    \end{split}
\end{align}
with
\begin{align}
    \hat\vz &= \argmax_{\vz \in \mathcal{Z}} S(\vz\mid x;\theta) + \gamma \label{eq:dlm-gumbel-max-z}\\
    \hat\vz(\eps) &= \argmax_{\vz \in \mathcal{Z}} S(\vz\mid x;\theta) + \gamma + \epsilon L(g(\vz)) \label{eq:dlm-gumbel-max-zeps}
\end{align}
\end{theorem}
\begin{proof}
We refer readers to the supplementary material of \citet{lorberbom2019direct}.
\end{proof}

The expansion in \Cref{eq:dlm-gumbel-max} circumvents the need to take the gradient through the $\argmax$ that is introduced by Gumbel-Max in \Cref{eq:reparam}. Intuitively, the gradient is estimated by slightly perturbing $\hat\vz$.

While \Cref{eq:dlm-gumbel-max} holds theoretically, one cannot implement the limit easily in a machine learning system. \citet{lorberbom2019direct} hence treated $\eps$ as a fixed hyperparameter that controls the bias-variance trade-off. Since too small an $\eps$ could lead to very large gradients, they only considered $\eps\ge0.1$. This, therefore, leads to a biased estimator for the true gradient. Nevertheless, they compared their method with Gumbel-softmax and showed that $\eps$ better controls the bias-variance trade-off than the temperature parameter in Gumbel-softmax.

In the structured case, direct loss minimization suffers from the same intractability as Gumbel-Max. \citet{lorberbom2019direct} therefore only considered pairwise structures. We believe attempting to combine direct loss minimization and methods such as Perturb-and-Parse would be an interesting future study.
\section{Applications} \label{sec:examples}

The methods introduced above have been widely adopted for optimizing models with latent discrete decisions. We review some applications below. We first discuss two settings where a downstream task is used as the training signal for latent structure induction, where the latent $z$ can either be a linguistic structure or an extractive summary of the input sentence. We also review when the target output is the same as the input, an auto-encoding setup, that allows unsupervised structure induction.

\subsection{Linguistic Structures}

As argued in \Cref{sec:intro}, latently inducing a linguistically-inspired structure that aids the downstream prediction can provide a useful inductive bias. Typically, a latent tree is induced from text, and some structure-encoding model such as TreeLSTM~\citep{tai-etal-2015-improved} or GCN~\citep{kipf2017semi} is used on top of this structure to compose the word-level representations into a sentence representation for the final prediction for various end tasks.

We presented graph-based parsing for bilexical dependencies in \Cref{sec:structure-prediction}. Transition-based parsers, on the other hand, uses the scoring function $S(\vz\mid x;\theta)=\sum_{\{\va : \operatorname{yield}(\va)=\vz\}}\prod_{i=1}^{\abs{\va}}f(a_i\mid x,\va_{<i};\theta)$, where $\operatorname{yield}(\cdot)$ is the mapping from a sequence of actions $\va=[a_1,\cdots,a_m]$, such as \textsc{shift} and \textsc{reduce}, to the deterministic structure produced by these actions. For constituency parsing, the default choice is the CKY algorithm~\citep{c,k,y}, a special case of chart parsing, which determines the optimal parse structure with dynamic programming by finding the optimal split point for spans.

\paragraph{Transition-based parsing and the CKY algorithm} are similar in that, in most settings, every decision sequence yields a valid structure. Therefore, we can simply backpropagate through the sequence of discrete decisions using the methods presented above. In this manner: \surrogatecitet{maillard-clark-2018-latent} trained a latent shift-reduce parser with STE. \relaxcitet{maillard_clark_yogatama_2019} induced a latent constituency structure with continuous relaxation using a $\softmax$ to normalize over all possible split points. \relaxcitet{bogin2021latent} used a similar architecture to improve compositional generalization for grounded question answering. \sfecitet{DBLP:conf/iclr/YogatamaBDGL17} learned a policy network for a latent dependency parser. \sfecitet{kim-etal-2019-unsupervised} trained an unsupervised recurrent neural network grammar~\citep{dyer-etal-2016-recurrent} that resembles a shift-reduce parser using reinforcement learning.
\reparamcitet{choi2018learning} latently induced a constituency tree by greedily choosing a pair of neighboring text spans to merge at each layer using straight-through Gumbel-softmax at each layer for sampling. \reparamcitet{hu-etal-2021-r2d2} used this method to design a heuristic pruning procedure that improves the $O(n^3)$ complexity of the CKY algorithm to $O(n)$. This efficiency also allowed them to pretrain this model with a language modeling objective.

\paragraph{Graph-based parsing} usually needs to be mindful of structural constraints. \surrogatecitet{peng-etal-2018-backpropagating} used SPIGOT to train latent syntactic and semantics parsers. \relaxcitet{kim2017structured} considered the expected local decisions by marginalizing over all possible structures, latently inducing a projective syntactic dependency tree structure (and also a linear sequential structure), which they termed ``structured attention.'' For example, if $Z_{pq}$ represents the existence of a syntactic edge between words $p$ and $q$ with $q$ being the parent, then one can obtain a context vector of $q$, $\vc_q$, by performing attention to the syntactic parent of $q$ with $\vc_q \propto \sum_p \pi(Z_{pq}=1 \mid x;\theta) \vx_p$ where $\vx_p$ denotes the representation of the word $p$. \relaxcitet{liu-etal-2018-structured} leveraged structured attention networks to induce latent constituency trees for richer inter-sentence attention~\citep{parikh-etal-2016-decomposable}. Structured attention networks, however, can be expensive during the marginalization step. For example, marginalizing over dependency trees using the inside-outside algorithm can result in a more than 10 times slowdown~\citep{liu-lapata-2018-learning} due to its lack of parallelizability. Additionally, it is inherently unable to consider non-projective dependency trees. To address these two issues, \relaxcitet{liu-lapata-2018-learning} proposed a generalization of structured attention networks by marginalizing using the Matrix-Tree Theorem~\citep{tutte1984graph} that allows the induction of non-projective dependency trees~\citep{koo-etal-2007-structured,smith-smith-2007-probabilistic,mcdonald-satta-2007-complexity} and with better parallelizability, achieving similar speed as their simple attention baseline. \relaxcitet{bisk-tran-2018-inducing} extended this method to machine translation and showed that it outperforms using simple attention as well as STE. \relaxcitet{niculae-etal-2018-towards} used SparseMAP to induce sparse distributions over latent dependency parses for sentence classification, natural language inference, and reverse dictionary lookup. \reparamcitet{zhou-etal-2020-amr} used a rectified distribution to induce latent dependency trees for syntactic-then-semantic AMR~\citep{banarescu-etal-2013-abstract}. \reparamcitet{corro-titov-2019-learning} used Perturb-and-Parse to obtain a latent dependency parser.

\subsection{Rationale Extraction}

One line of work concerns the interpretability of neural models and aims to extract the words and phrases in a span of text that are responsible for some task prediction, usually termed ``rationale extraction.'' \citet{lei-etal-2016-rationalizing} first introduced this task and considered the intermediate structure as a collection of Bernoulli variables, each corresponding to a word in the original text, with some independence assumption among them. A condensed summary is formed by collecting the words corresponding to the activated Bernoulli variables. Then, a predictor learns a mapping from the summary to the final prediction, which is expected to stay close to the gold label. The summary usually needs to be latently induced due to a lack of summary annotation. %

\sfecitet{lei-etal-2016-rationalizing} proposed to jointly train a rationale extractor and a downstream task predictor using score function estimators. They experimented with both an unstructured intermediate rationale where each word is independently selected as a Bernoulli random variable, as well a structured rationale where the selection of each word recurrently depends on previous selections. Their method achieved similar performance to using the full text while in most cases including only up to 30\% of the text. \reparamcitet{bastings-etal-2019-interpretable} used reparameterization with a rectified distribution and demonstrated improved performance.

\subsection{Auto-Encoder}

Another important body of work uses an autoencoding objective where the target $y$ is the same as the input $x$. A reconstruction loss is typically either maximized directly~\citep{ammar2014conditional} or, when it is intractable as is common with neural network parameterized models, using variational autoencoders (VAE;~\citealp{vae}). An encoder, or inference network, $f$ that models $p(\vz\mid x)$ and a reconstruction network $g$ that models $p(x\mid\vz)$ are jointly learned, after which we can take only the inference network to predict a structure from text.

With the output the same as the input, this allows completely unsupervised learning. Sometimes some amount of intermediate supervision is also added to ensure that the intermediate structure resembles some target formalism.

\paragraph{Linguistic Structures}
When $\vz$ is a linguistic structure, this method enables unsupervised or semi-supervised parsing.
\relaxcitet{drozdov-etal-2019-unsupervised-latent,drozdov-etal-2020-unsupervised} and \relaxcitet{xu-etal-2021-improved} used a continuous relaxation based approach for unsupervised and distantly-supervised parsing. \relaxcitet{shen2018neural} softly and latently attends to a word's syntactic siblings using proximity values induced by a ``syntactic distance'' for language modeling and unsupervised constituency parsing. \sfecitet{yin-etal-2018-structvae} used reinforcement learning to train a semi-supervised semantic parser. \reparamcitet{corro2018differentiable} used Perturb-and-Parse for semi-supervised parsing.

\paragraph{Others}
The methods reviewed in this survey have also been used in VAEs when $\vz$ is not a linguistic structure.
For example, \sfecitet{miao-blunsom-2016-language} considered $\vz$ as a compressed version of the input sentence and trained a semi-supervised sentence compression model.
Using Gumbel-softmax, \reparamcitet{zhou-neubig-2017-multi} sampled independent morphological label as $\vz$ for morphological re-inflection, and
\reparamcitet{rezaee-ferraro-2021-event} treated semantic frames as $\vz$ and trained a model for semi-supervised event modeling.
\section{What is Latently Learned?} \label{sec:analysis}

We now inspect the structure that is latently learned. We focus on latent syntax-inspired trees as they are easy to compare to human's syntactic judgment and there are treebanks, or silver parser-generated trees, for reference.

Perhaps surprisingly, despite the strong downstream performance of these latent structure models, the trees that many such models learn for the most part do not correspond to traditional linguistic formalisms and at best capture only shallow syntactic units. \citet{williams-etal-2018-latent} examined the reinforcement learning based model of \citet{DBLP:conf/iclr/YogatamaBDGL17} (RL-SPINN) and the straight-through Gumbel based model of \citet{choi2018learning} (ST-Gumbel). They took these models that are trained on NLI datasets and evaluated them on the Penn Treebank (PTB; \citealp{marcus-etal-1993-building}).
RL-SPINN achieved similar parsing performance as random trees, while ST-Gumbel was even worse. In particular, RL-SPINN's induced trees are highly similar to a purely left-branching structure with an \fone of $>99\%$. In contrast to this deepest possible strategy, the trees of ST-Gumbel are very shallow with an average depth of 4.2, close to 3.9 for balanced trees, and much shallower than silver parses with a 5.7 average depth. \citet{maillard-clark-2018-latent} confirmed that the models of themselves and of \citet{maillard_clark_yogatama_2019} do not induce trees that resemble those produced by the Stanford Parser. \citet{bisk-tran-2018-inducing} similarly showed that their method achieved an attachment score similar to or even worse than flat baselines.

\citet{nangia-bowman-2018-listops} introduced the ListOps dataset for synthetic math expression evaluation which is easy to solve with a parse structure but difficult otherwise. Both RL-SPINN and ST-Gumbel even underperform a simple LSTM model that has no structured modeling, suggesting that these models, with the datasets and training algorithms they used, cannot learn the structure required for this task. Nevertheless, \citet{havrylov-etal-2019-cooperative}, which improved over RL-SPINN as introduced in \Cref{sec:sfe}, achieved near-perfect performance on this dataset by reducing the gradient estimation variance and reducing the coadaptation issue of the parser and the downstream predictor.

Overall, it seems that for the latently induced structure to bear a close resemblance to existing syntactic formalisms is not a necessary condition for strong downstream task performance. In fact, many studies showed that, in latent structure learning, using parses from a pretrained parser deteriorates the performance~\citep{choi2018learning,maillard_clark_yogatama_2019}, pretraining on these parses does not help~\citep{kim2017structured}, and neither does including an auxiliary parsing loss~\citep{DBLP:conf/iclr/YogatamaBDGL17}. Some studies qualitatively inspected the latent structure and discovered deviation from existing formalisms that may benefit composition~\citep{bisk-tran-2018-inducing,niculae-etal-2018-towards}. For example, in coordination constructions, Universal Dependencies dictate the leftmost conjunct as the head~\citep{ud}, while \citet{niculae-etal-2018-towards} argued that a symmetric treatment may provide a better composition order. On the other hand, \citet{shi-etal-2018-tree} noted that trivial trees (i.e., balanced or linear trees) achieve performance similar to or better than latently induced trees or external parse trees on a variety of classification and generation tasks, casting doubt on the real source of performance improvement. More future work is needed to understand the role syntax plays in this picture and how it might vary across different tasks.

The story is different for methods for unsupervised parsing which can induce familiar linguistic formalisms with decent accuracy (though see \citet{li-risteski-2021-limitations} for a theoretical analysis of their limitations). These methods usually use a language modeling or autoencoding objective rather than receiving supervision from a natural language understanding task~\citep[\emph{inter alia}]{shen2018neural,kim-etal-2019-unsupervised,hu-etal-2021-r2d2}. Why these objectives are more conducive to the induction of a latent structure that better resembles existing linguistic formalisms would be an interesting future research question.
\section{Conclusion}

In this work, we surveyed three families of methods that tackle the non-differentiability of discrete decisions caused by the $\argmax$ function and introduced how they respectively handle the unique challenges of structure prediction. These methods and some of their key properties are summarized in \Cref{tab:summary}. We also reviewed past work that studied the induced latent structures and highlighted that they mostly do not closely resemble existing linguistic formalisms. We hope this survey can serve to catalyze future research in this area.

\section*{Acknowledgment}

We thank Greg Durrett, Andr\'{e} Martins, Vlad Niculae, Hao Peng, Noah Smith, Nishant Subramani, ordered alphabetically by last name, for helpful discussions and feedback on various versions of this work.

\bibliographystyle{acl_natbib}
\bibliography{project}
\end{document}